\documentclass[journal=jctcce,manuscript=article]{achemso}

\usepackage{graphicx}
\usepackage{bm}
\usepackage{physics}
\usepackage{algorithm}
\usepackage{algpseudocode}
\usepackage{tabularx}
\usepackage{array}
\usepackage{hyperref}
\usepackage{tikz}
\usepackage{braket}
\usepackage{booktabs}
\usepackage{amsthm}
\usepackage{amsmath}
\usepackage{amssymb}
\usetikzlibrary{arrows.meta, positioning, backgrounds, fit}

\newtheorem{theorem}{Theorem}
\newtheorem{proposition}{Proposition}

\author{Sk Mujaffar Hossain}
\affiliation{Indo-Korea Science and Technology Center (IKST),
  Bengaluru 560064, India}

\author{Satadeep Bhattacharjee}
\email{s.bhattacharjee@ikst.res.in}
\affiliation{Indo-Korea Science and Technology Center (IKST),
  Bengaluru 560064, India}

\title{Optimal Filtered Spectral Projection for Quantum Principal
Component Analysis}

\begin{document}

\begin{abstract}
Quantum principal component analysis (qPCA) is commonly formulated as the
extraction of eigenvalues and eigenvectors of a covariance-encoded density
operator. Yet in many qPCA settings the practical goal is simpler: projection
onto the dominant spectral subspace. Here we introduce a projection-first
framework, the Filtered Spectral Projection Algorithm (FSPA), which bypasses
explicit eigenvalue estimation while preserving the relevant spectral
structure. FSPA amplifies any nonzero warm-start overlap with the leading
subspace and remains robust in small-gap and near-degenerate regimes, without
artificial symmetry breaking in the absence of bias. We show that FSPA achieves
an oracle complexity $\mathcal{O}((\log(1/\epsilon)+\log(1/|a_1|^2))/\log(\lambda_1/\lambda_2))$,which is tight by a matching lower bound, establishing it as an\emph{optimal} projection primitive. We derive a convergence rate for
degenerate spectra, give a circuit resource analysis with $n+\mathcal{O}(1)$
qubit overhead independent of system dimension, and extend the method to
threshold spectral projection, Threshold-FSPA, which converges in
$\mathcal{O}(\log(1/\epsilon))$ calls when the threshold lies between
eigenvalues. In the density matrix exponentiation access model, FSPA gives an
exponential copy-complexity advantage over classical methods. For classical
datasets, we show that for amplitude-encoded centered data the ensemble density
matrix $\rho=\sum_i p_i|\psi_i\rangle\langle\psi_i|$ equals the covariance
matrix. Numerical tests on chemistry density matrices, noisy circuit outputs,
Breast Cancer Wisconsin, handwritten Digits, and 1--4-qubit scalability confirm
the theory. A minimal Qiskit implementation validates magnitude invariance,
signal amplification, and no spurious symmetry breaking. These results
establish FSPA as an optimal and deployable quantum spectral projection
primitive.
\end{abstract}


\section{Introduction}

Principal component analysis (PCA) is among the most widely used primitives
in classical data analysis, underpinning dimensionality reduction, covariance
estimation, and feature extraction across scientific computing and machine
learning.\cite{golub2013matrix} In the quantum setting, PCA acquires new
significance: a quantum state can encode high-dimensional data through amplitude
encoding, and the covariance structure of an ensemble of such states is captured
by a density operator $\rho$ that is both the object of analysis and a
computational resource. Quantum principal component analysis (qPCA) is one of
the foundational ideas in quantum machine learning (QML) because it connects
quantum linear-algebra methods to dimensionality reduction, feature extraction,
and covariance-structure analysis in high-dimensional data.\cite{lloyd2014quantum,lloyd2013quantum,gordon2022covariance,he2022low,lin2019improved}
Its importance is both conceptual and algorithmic: the seminal
Lloyd--Mohseni--Rebentrost construction shows that copies of a density matrix
can be used to implement $e^{-i\rho t}$ and, via phase estimation, extract
principal eigendirections in quantum form,\cite{lloyd2014quantum} reframing
quantum states as active resources in their own analysis.

Despite its foundational role, qPCA as originally formulated faces two distinct
and often conflated limitations. The first is \emph{spectral gap dependence}:
resolving nearby eigenvalues requires phase estimation time
$t \sim 1/(\lambda_1 - \lambda_2)$, imposing fundamental lower bounds on
convergence rates.\cite{lloyd2014quantum} The second, more subtle limitation
is \emph{eigenvalue magnitude dependence}: even when the spectral gap is fixed,
algorithms based on phase encoding fail when the absolute scale of eigenvalues
is small, because the phase signal $e^{-i\lambda t}$ becomes indistinguishable
from the identity as $\lambda \to 0$ at finite evolution time.\cite{nghiem2025new}
Nghiem et al.\ demonstrated rigorously that this magnitude failure mode renders
standard qPCA ineffective even when the spectral ordering is perfectly
preserved.\cite{nghiem2025new} A third consideration, highlighted by Tang's
dequantization result,\cite{tang2021} is that complexity claims for qPCA must
be interpreted carefully with respect to the data-access model. Together, these
points motivate a sharper formulation of the task before choosing the spectral
primitive.

Our central observation is that \emph{in most practical qPCA settings, the
operative target is spectral projection, not eigenvalue estimation}. When
leading eigenvalues are degenerate or nearly degenerate --- as occurs in quantum
chemistry, quantum state tomography, and covariance-dominated learning tasks
--- individual eigenvectors are not uniquely defined and are highly sensitive to
perturbations. The dominant invariant subspace, by contrast, is stable and
physically meaningful. A projection-first algorithm should therefore amplify
overlap with this subspace without requiring eigenvalue resolution, and without
failing when the global eigenvalue scale is small.

We introduce the \emph{Filtered Spectral Projection Algorithm} (FSPA), an
adaptive filter-and-renormalize iteration that amplifies dominant spectral
support without explicitly encoding eigenvalues. At the state-update level,
FSPA is equivalent to normalized power iteration with an adaptive schedule;
we therefore do not claim improved gap-scaling asymptotics over that
baseline.\cite{golub2013matrix} The contribution is instead task-level: a
projection-first quantum spectral primitive with built-in magnitude invariance
and explicit handling of degenerate dominant spectra, naturally interpretable
in block-encoding/QSVT language.\cite{nghiem2025new,gilyen2019quantum} FSPA
is not intended to outperform classical power iteration in terms of asymptotic
complexity. Its role is to serve as a stable quantum spectral primitive within
circuits that already operate on quantum-encoded states --- for instance, as a
subroutine in a larger quantum algorithm where $\rho$ is accessed via
block-encoding rather than as an explicit classical matrix. In this setting,
the renormalization step and magnitude invariance are properties of the quantum
circuit itself, not of a classical computation.

Our main theoretical contributions establish a complete picture of FSPA's
complexity. \textbf{(i)} FSPA is invariant under uniform spectral rescaling,
directly eliminating the failure mode of Nghiem et al. \textbf{(ii)} FSPA
achieves oracle complexity
$\mathcal{O}((\log(1/\epsilon)+\log(1/|a_1|^2))/ \log(\lambda_1/\lambda_2))$,
and this bound is \emph{tight}: a matching lower bound proves FSPA is an
optimal projection primitive. \textbf{(iii)} In degenerate settings, FSPA
converges to the dominant invariant subspace with a quantitative rate governed
by the gap between the subspace and the rest of the spectrum. \textbf{(iv)} A
circuit resource analysis establishes $n + \mathcal{O}(1)$ qubit overhead
independent of dimension. \textbf{(v)} In the density matrix exponentiation
model, FSPA achieves exponential copy-complexity advantage over classical
methods. \textbf{(vi)} Threshold-FSPA extends the primitive to
eigenvalue-threshold subspaces. Numerical experiments on quantum chemistry
1-RDMs, noisy circuits, and a 1--4 qubit scalability study confirm all
predictions.


\section{Related Work}
\label{sec:related}

\textit{Quantum phase estimation and qPCA.}---
Quantum phase estimation (QPE)\cite{kitaev1995quantum} is the foundational
algorithm for eigenvalue extraction. Lloyd et al.\cite{lloyd2014quantum} applied
QPE to density matrix exponentiation to obtain the first qPCA algorithm. Nghiem
et al.\cite{nghiem2025new} rigorously identified eigenvalue magnitude dependence
as a distinct failure mode: small absolute eigenvalues cause phase signal
collapse even at fixed spectral ordering. FSPA is designed to eliminate this
failure mode.

\textit{QSVT and block-encoding.}---
The quantum singular value transformation (QSVT) framework of Gily\'{e}n et
al.\cite{gilyen2019quantum} provides a unifying language for quantum linear
algebra, showing that any polynomial function of a block-encoded operator can
be implemented with near-optimal query complexity. Low and Chuang\cite{low2019hamiltonian}
developed Hamiltonian simulation by qubitization, which underlies many
block-encoding constructions. FSPA is naturally interpretable within QSVT:
each doubling round implements a low-degree polynomial spectral filter, and
the adaptive schedule generates increasing polynomial degrees. Unlike standard
QSVT constructions based on a fixed precompiled polynomial, FSPA employs an
adaptive schedule with renormalization, which enables magnitude invariance.

\textit{Dequantization.}---
Tang\cite{tang2021} showed that quantum-inspired classical algorithms can match
qPCA performance under the quantum-sampling access model, demonstrating that
part of qPCA's apparent exponential advantage arises from the access model.
This result applies to the sampling access model and does not affect FSPA's
exponential advantage in the density matrix exponentiation model, where $\rho$
is available only as quantum copies.

\textit{Classical subspace iteration.}---
Classical power iteration and subspace iteration\cite{golub2013matrix} are the
direct classical analogues of FSPA. Krylov subspace methods\cite{saad2011numerical}
achieve faster convergence for certain spectral distributions but require storing
an orthogonalized basis of increasing dimension, which is not straightforward in
the quantum setting. FSPA trades this faster convergence for the simplicity and
block-encoding compatibility required by quantum circuits.

\textit{Quantum chemistry.}---
Reduced density matrices play a central role in quantum chemistry, where the
one-body reduced density matrix (1-RDM) encodes natural orbital occupation
numbers.\cite{helgaker2000molecular} Recent work on variational quantum
eigensolvers\cite{peruzzo2014variational} has highlighted the need for stable
spectral primitives for quantum-classical hybrid algorithms. Our experiments on
H$_2$, LiH, and BeH$_2$ demonstrate FSPA in this context.


\section{Filtered Spectral Projection Algorithm}
\label{sec:theory}

\textit{Block-encoding framework.}---
To make the circuit-level cost precise, we adopt the standard block-encoding
framework.\cite{gilyen2019quantum} A Hermitian operator $\rho$ acting on $n$
qubits ($d = 2^n$) is $(\kappa, \delta)$-block-encoded by a unitary $U_\rho$
acting on $n + \kappa$ qubits such that

\begin{equation}
  \bigl(\langle 0|^{\otimes\kappa} \otimes I_n\bigr)\,
  U_\rho\,
  \bigl(|0\rangle^{\otimes\kappa} \otimes I_n\bigr)
  = \rho + \Delta, \quad \|\Delta\| \le \delta,
  \label{eq:block_encoding}
\end{equation}

where $\kappa$ ancilla qubits carry the block-encoding overhead. Each FSPA
oracle call --- one application of $\rho$ followed by renormalization --- is
implemented by one call to $U_\rho$ and one post-selection on the ancilla
register.

\textit{Problem setting.}---
Let
\[
\rho = \sum_{j=1}^d \lambda_j |\psi_j\rangle\langle\psi_j|
\]
be a Hermitian operator with eigenvalues ordered as
$\lambda_1 \ge \lambda_2 \ge \cdots \ge \lambda_d \ge 0$. Given an input state
$|\phi_0\rangle=\sum_j a_j|\psi_j\rangle$, we distinguish four objectives:

\begin{itemize}
\item \textbf{Top-eigenvector recovery:} output a state close to
  $|\psi_1\rangle$ when $\lambda_1>\lambda_2$.
\item \textbf{Dominant-subspace recovery:} output a state (or projector)
  supported on
  $\mathcal{S}_R=\mathrm{span}\{|\psi_1\rangle,\ldots,|\psi_R\rangle\}$
  when $\lambda_1=\cdots=\lambda_R>\lambda_{R+1}$.
\item \textbf{Overlap amplification:} increase
  $\langle\phi|P_{\mathcal{S}_R}|\phi\rangle$ from nonzero initial bias.
\item \textbf{Eigenvalue estimation:} approximate $\lambda_j$ to prescribed
  precision.
\end{itemize}

FSPA targets dominant-subspace recovery and overlap amplification directly
rather than eigenvalue estimation.

\begin{algorithm}[H]
\caption{Filtered Spectral Projection Algorithm (FSPA)}
\label{alg:fspa}
\begin{algorithmic}[1]
\Require Hermitian operator $\rho$ with $\|\rho\| \le 1$, initial
  state $|\phi_0\rangle$, number of rounds $T$
\Ensure State with amplified overlap on dominant eigenspace
\State Normalize $|\phi\rangle \gets |\phi_0\rangle / \|\phi_0\|$
\Comment{Normalization eliminates uniform eigenvalue rescaling dependence}
\State Set amplification parameter $\beta \gets 1$
\For{$t = 1$ to $T$}
    \For{$k = 1$ to $\beta$}
        \State $|\phi\rangle \gets \rho |\phi\rangle$
        \State Normalize $|\phi\rangle \gets |\phi\rangle / \|\phi\|$
    \EndFor
    \State $\beta \gets 2\beta$
\EndFor
\State \Return $|\phi\rangle$
\end{algorithmic}
\end{algorithm}

The adaptive growth of $\beta$ yields progressively stronger filtering while
the renormalization step removes global eigenvalue-scale
dependence.\cite{golub2013matrix} At the state-update level, FSPA is
mathematically equivalent to normalized power iteration applied with an
adaptive schedule. This relation is intentional, not a hidden limitation.
What is standard is the gap-limited amplification mechanism; what is new here
is the projection-first formulation in qPCA settings where magnitude collapse
of estimation-first methods is relevant.\cite{nghiem2025new}

From a polynomial viewpoint, each round applies a low-degree spectral filter
to $\rho$, and the adaptive schedule generates a sequence of increasing
effective degrees. This interpretation aligns with block-encoding/QSVT
realizations of polynomial operator
transforms.\cite{gilyen2019quantum,low2019hamiltonian}

\textit{Structural properties.}---
The normalized iteration underlying FSPA has several structural properties.

\begin{proposition}[Eigenvalue Magnitude Invariance]
\label{prop:magnitude_invariance}
Let $\rho$ be a Hermitian operator and let $c>0$. The normalized iterates
produced by FSPA applied to $\rho$ and $c\rho$ are identical at every step.
In particular, FSPA is invariant under uniform rescaling of the spectrum.
\end{proposition}

This first property is immediate from renormalization but central to the
interpretation of FSPA as a projection-first primitive.

\begin{proposition}[Gap-Dependent Bias Amplification]
\label{prop:gap_amplification}
Let $\rho$ be Hermitian with ordered eigenvalues $\lambda_1 \ge \lambda_2 \ge \cdots$.
If $\lambda_1 > \lambda_2$ and the initial state has nonzero overlap with the
dominant eigenvector, then the overlap produced by FSPA increases monotonically
with the number of rounds. The convergence rate is governed by
$\lambda_1/\lambda_2$, equivalently by the spectral gap. If
$\lambda_1 = \lambda_2$, convergence is only to the corresponding invariant
subspace.
\end{proposition}

\begin{proposition}[Subspace Convergence]
\label{prop:subspace_convergence}
Let the dominant eigenspace of $\rho$ have degeneracy $R \ge 1$, i.e.,
$\lambda_1 = \cdots = \lambda_R > \lambda_{R+1}$. Then, in the limit of
infinite amplification rounds $T \to \infty$, FSPA produces a normalized state
whose fidelity with the dominant invariant subspace $\mathcal{S}_R$ approaches
unity. Convergence to a unique eigenvector is not guaranteed unless degeneracy
is lifted. A quantitative convergence rate is given in
Theorem~\ref{thm:subspace_complexity}.
\end{proposition}

These propositions clarify the roles of scale invariance, spectral ratios, and
degeneracy. The main quantitative guarantee is the following.

\begin{theorem}[Oracle Complexity of FSPA]
\label{thm:complexity}
Let $\rho$ be a Hermitian operator with spectral decomposition
$\rho = \sum_{j=1}^d \lambda_j |\psi_j\rangle\langle\psi_j|$, where
$\lambda_1 > \lambda_2 \ge \cdots \ge 0$. Let the initial state be
$|\phi_0\rangle = \sum_{j=1}^d a_j |\psi_j\rangle$ with $a_1 \neq 0$. Define
the spectral ratio $r := \lambda_2/\lambda_1 < 1$. Then FSPA produces a state
$|\phi_k\rangle$ satisfying $|\langle\psi_1|\phi_k\rangle|^2 \ge 1 - \epsilon$
after at most
\[
  \mathcal{O}\!\left(
    \frac{\log(1/\epsilon) + \log(1/|a_1|^2)}
         {\log(\lambda_1/\lambda_2)}
  \right)
\]
applications of $\rho$.
\end{theorem}

\textit{Proof sketch.}
We expand the initial state in the eigenbasis and apply $\rho^k$:
\[
  \rho^k|\phi_0\rangle
  = a_1\lambda_1^k|\psi_1\rangle
  + \sum_{j\ge2} a_j\lambda_j^k|\psi_j\rangle
  = \lambda_1^k\!\left(
      a_1|\psi_1\rangle + \sum_{j\ge2} a_j r_j^k|\psi_j\rangle
    \right),
\]
where $r_j = \lambda_j/\lambda_1 \le r < 1$. After normalization,
\[
  F_k = \frac{|a_1|^2}{|a_1|^2 + \sum_{j\ge2}|a_j|^2 r_j^{2k}}
      \ge \frac{|a_1|^2}{|a_1|^2 + (1-|a_1|^2)r^{2k}}.
\]
Imposing $F_k \ge 1-\epsilon$ and solving gives
\[
  k \ge \frac{\log\!\bigl((1-\epsilon)\epsilon^{-1}(1-|a_1|^2)|a_1|^{-2}\bigr)}
             {2\log(\lambda_1/\lambda_2)},
\]
yielding the stated complexity. The adaptive doubling schedule achieves the
required $k$ oracle calls using $T = \lceil\log_2(k+1)\rceil$ rounds, since
the total calls after $T$ rounds is $2^T - 1 \le 2k$, a factor-of-two overhead
with no change in asymptotic complexity. \hfill$\square$

\begin{figure}[t]
    \centering
    \includegraphics[width=\linewidth]{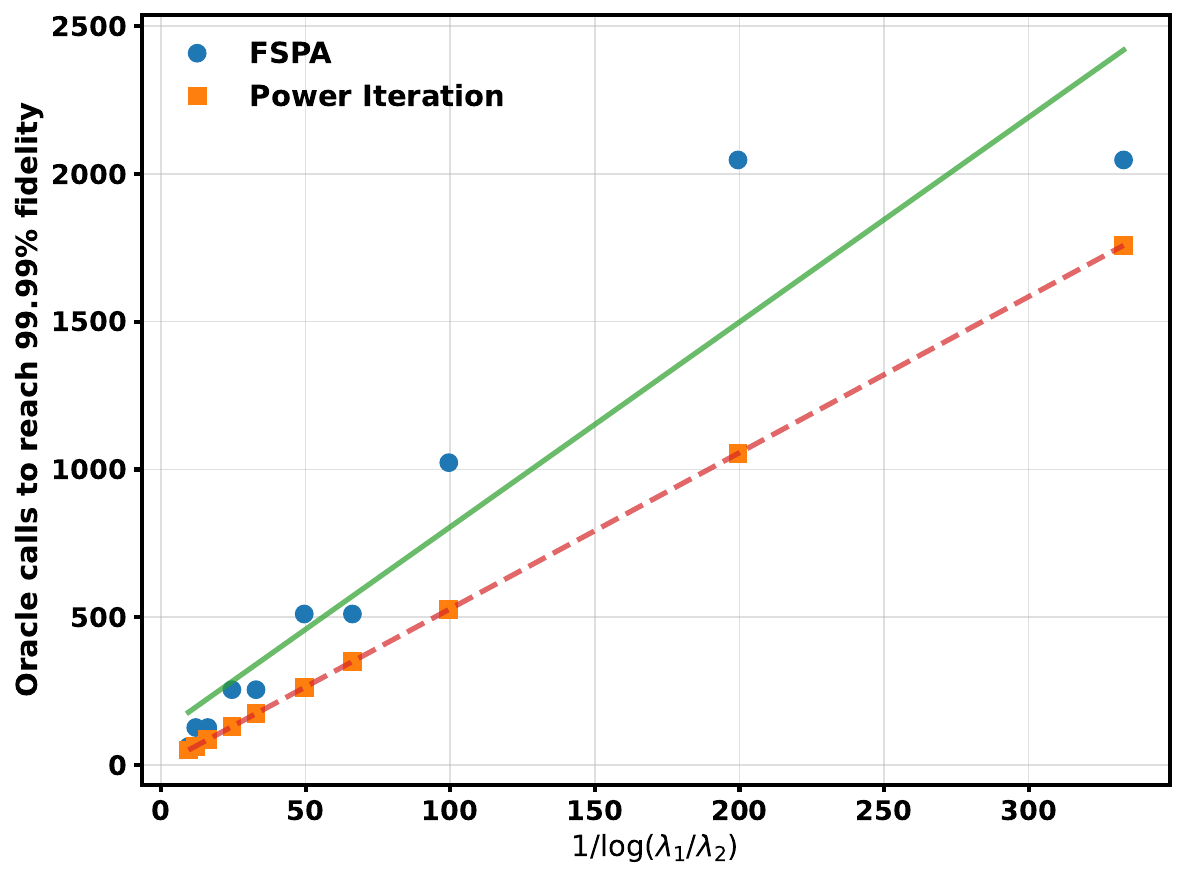}
    \caption{Empirical validation of the gap-dependent oracle complexity of
    FSPA. The total number of oracle applications required to reach 99.99\%
    fidelity is plotted against the theoretical scaling variable
    $1/\log(\lambda_1/\lambda_2)$. Linear regression confirms proportional
    scaling, consistent with the predicted complexity
    $\mathcal{O}(\log(1/\epsilon)/\log(\lambda_1/\lambda_2))$. Classical power
    iteration is shown for comparison.}
    \label{fig:gap_scaling_validation}
\end{figure}

Figure~\ref{fig:gap_scaling_validation} strengthens the complexity claim
empirically. The near-linear trend against $1/\log(\lambda_1/\lambda_2)$
confirms that the empirical oracle count is controlled by spectral-ratio
amplification, not by absolute eigenvalue magnitude. FSPA and classical power
iteration follow comparable slopes under matched initialization, indicating
that FSPA preserves the same gap-law asymptotics while shifting the algorithmic
objective toward projection robustness.

\textit{Circuit resources.}---

\begin{proposition}[Circuit Resources of FSPA]
\label{prop:circuit_resources}
Let $\rho$ act on $n$ qubits ($d = 2^n$) and suppose $\rho$ is
$(\kappa,\delta)$-block-encoded by a unitary $U_\rho$ with gate complexity
$\mathcal{G}(U_\rho)$ and depth $\mathcal{D}(U_\rho)$. Then FSPA produces a
state with fidelity $\ge 1-\epsilon$ with the dominant eigenvector using:

\begin{itemize}
  \item \textbf{Qubits:} $n + \kappa + \mathcal{O}(1)$,
  \item \textbf{Calls to $U_\rho$:}
    $\mathcal{O}((\log(1/\epsilon)+\log(1/|a_1|^2))/
    \log(\lambda_1/\lambda_2))$,
  \item \textbf{Gate count:}
    $\mathcal{O}((\log(1/\epsilon)+\log(1/|a_1|^2))
    \cdot\mathcal{G}(U_\rho)/\log(\lambda_1/\lambda_2))$,
  \item \textbf{Circuit depth:}
    $\mathcal{O}((\log(1/\epsilon)+\log(1/|a_1|^2))
    \cdot\mathcal{D}(U_\rho)/\log(\lambda_1/\lambda_2))$.
\end{itemize}
\end{proposition}

\begin{proof}
The call count follows from Theorem~\ref{thm:complexity}. Each call to the
oracle uses one call to $U_\rho$ acting on $n+\kappa$ qubits, contributing
$\mathcal{G}(U_\rho)$ gates and depth $\mathcal{D}(U_\rho)$. The
renormalization step requires one additional ancilla qubit and $\mathcal{O}(1)$
gates per call for post-selection.
\end{proof}

\noindent\textit{Remark.}
For a sparse $d\times d$ density matrix with at most $s$ nonzero entries per
row, standard sparse block-encoding constructions yield
$\mathcal{G}(U_\rho)=\mathcal{O}(s\,\mathrm{polylog}(d))$ and
$\kappa=\mathcal{O}(\log d)$.\cite{gilyen2019quantum} The gate count is
therefore
$\mathcal{O}(s\,\mathrm{polylog}(d) \cdot(\log(1/\epsilon)+\log(1/|a_1|^2))/\log(\lambda_1/\lambda_2))$,
independent of the system dimension $d$ in the prefactor.

\textit{Optimality.}---
The upper bound in Theorem~\ref{thm:complexity} is tight.

\begin{proposition}[Oracle Complexity Lower Bound]
\label{prop:lower_bound}
Let $\rho$ be Hermitian with $\lambda_1 > \lambda_2$ and $r = \lambda_2/\lambda_1$.
Any algorithm using oracle access to $\rho$ that produces a state
$|\phi_k\rangle$ satisfying $|\langle\psi_1|\phi_k\rangle|^2 \ge 1 - \epsilon$
from initial overlap $|a_1|^2$ requires at least
\[
  \Omega\!\left(
    \frac{\log(1/\epsilon) + \log(1/|a_1|^2)}
         {\log(\lambda_1/\lambda_2)}
  \right)
\]
oracle calls.
\end{proposition}

\begin{proof}
Define the log-odds $L_k = \log(F_k/(1-F_k))$ where
$F_k = |\langle\psi_1|\phi_k\rangle|^2$. To reach $F_k \ge 1-\epsilon$ from
$F_0 = |a_1|^2$, the quantity $L_k$ must increase by
$\Delta L = \Omega(\log(1/\epsilon)+\log(1/|a_1|^2))$. The most favorable
single-step update $|\phi\rangle\to\rho|\phi\rangle/\|\rho|\phi\rangle\|$
satisfies $F_{k+1}/(1-F_{k+1}) \le r^{-2} \cdot F_k/(1-F_k)$, so each oracle
call increases $L_k$ by at most $2\log(1/r) = 2\log(\lambda_1/\lambda_2)$.
Dividing $\Delta L$ by this maximum per-step increase gives the lower bound.
\end{proof}

\noindent Together, Theorem~\ref{thm:complexity} and
Proposition~\ref{prop:lower_bound} establish FSPA as an \emph{optimal}
projection primitive in the spectral-ratio complexity model.

\textit{Subspace convergence rate.}---

\begin{theorem}[Subspace Oracle Complexity]
\label{thm:subspace_complexity}
Let $\lambda_1 = \cdots = \lambda_R > \lambda_{R+1}$, let $P_{\mathcal{S}_R}$
be the projector onto the dominant invariant subspace, and let
$\alpha := \|P_{\mathcal{S}_R}|\phi_0\rangle\|^2 > 0$ and
$r_{\mathcal{S}} := \lambda_{R+1}/\lambda_1 < 1$. Then FSPA produces a state
satisfying $\|P_{\mathcal{S}_R}|\phi_k\rangle\|^2 \ge 1-\epsilon$ after at
most
\[
  \mathcal{O}\!\left(
    \frac{\log(1/\epsilon) + \log(1/\alpha)}
         {\log(\lambda_1/\lambda_{R+1})}
  \right)
\]
oracle calls.
\end{theorem}

\begin{proof}
Applying $\rho^k$ and decomposing into subspace and orthogonal complement
components:
$\rho^k|\phi_0\rangle = \lambda_1^k P_{\mathcal{S}_R}|\phi_0\rangle
+ \sum_{j>R}a_j\lambda_j^k|\psi_j\rangle$. After normalization, the subspace
fidelity satisfies
$\|P_{\mathcal{S}_R}|\phi_k\rangle\|^2 \ge \alpha/(\alpha
+ (1-\alpha)r_{\mathcal{S}}^{2k})$. Imposing this $\ge 1-\epsilon$ and solving
for $k$ gives the stated bound.
\end{proof}

\noindent\textit{Remark.}
Theorem~\ref{thm:subspace_complexity} reduces to Theorem~\ref{thm:complexity}
in the non-degenerate case $R=1$. For $R>1$, the relevant gap is
$\lambda_1-\lambda_{R+1}$ (the subspace boundary gap), which is well-defined
and stable even when individual eigenvalue spacings within $\mathcal{S}_R$
vanish.

\textit{Warm-start oracle reduction.}---

\begin{proposition}[Warm-Start Oracle Reduction]
\label{prop:warmstart}
If the initial state satisfies $|\langle\psi_1|\phi_0\rangle|^2 \ge \delta$,
then FSPA achieves fidelity $1-\epsilon$ using
$\mathcal{O}((\log(1/\epsilon)+\log(1/\delta))/\log(\lambda_1/\lambda_2))$
oracle calls. Compared to a uniformly random initial state with
$|\langle\psi_1|\phi_0\rangle|^2 \approx 1/d$, the warm start reduces the
oracle count by a factor of
$(\log(1/\epsilon)+\log d)/(\log(1/\epsilon)+\log(1/\delta))$, which is
$\Omega(\log d/\log(1/\delta))$ when $\delta \gg 1/d$.
\end{proposition}

Table~\ref{tab:results_summary} summarises all theoretical results.

\begin{table*}
\centering
\caption{Summary of all theoretical results established in this work.
$r = \lambda_2/\lambda_1$; $\epsilon$ = target infidelity; $\alpha$ = initial subspace
overlap; $r_S = \lambda_{R+1}/\lambda_1$; $\Delta_\tau$ = normalised eigenvalue-threshold
gap; $\kappa$ = block-encoding ancilla overhead; $d = 2^n$ = system dimension.}
\label{tab:results_summary}
\small
\begin{tabularx}{\linewidth}{llXc}
\toprule
\textbf{Result} & \textbf{Type} & \textbf{Key statement} & \textbf{Section} \\
\midrule
Magnitude invariance
  & Proposition~\ref{prop:magnitude_invariance}
  & Iterates of $\rho$ and $c\rho$ are identical
  & \ref{sec:theory} \\[2pt]
Gap-dependent amplification
  & Proposition~\ref{prop:gap_amplification}
  & Overlap monotone; rate $\propto \lambda_1/\lambda_2$
  & \ref{sec:theory} \\[2pt]
Subspace convergence (qualitative)
  & Proposition~\ref{prop:subspace_convergence}
  & $\|P_{\mathcal{S}_R}|\phi_T\rangle\|^2 \to 1$ as $T \to \infty$
  & \ref{sec:theory} \\[2pt]
Oracle complexity (upper bound)
  & Theorem~\ref{thm:complexity}
  & $\mathcal{O}((\log(1/\epsilon) + \log(1/|a_1|^2))/\log(1/r))$
  & \ref{sec:theory} \\[2pt]
Circuit resources
  & Proposition~\ref{prop:circuit_resources}
  & $n{+}\kappa{+}\mathcal{O}(1)$ qubits; gate count $\propto \mathcal{G}(U_\rho)/\log(1/r)$
  & \ref{sec:theory} \\[2pt]
Oracle complexity (lower bound)
  & Proposition~\ref{prop:lower_bound}
  & $\Omega((\log(1/\epsilon) + \log(1/|a_1|^2))/\log(1/r))$ (optimal)
  & \ref{sec:theory} \\[2pt]
Subspace convergence (quantitative)
  & Theorem~\ref{thm:subspace_complexity}
  & $\mathcal{O}((\log(1/\epsilon) + \log(1/\alpha))/\log(\lambda_1/\lambda_{R+1}))$
  & \ref{sec:theory} \\[2pt]
Warm-start reduction
  & Proposition~\ref{prop:warmstart}
  & Bias $\delta$ saves $\mathcal{O}(\log(d/\delta^{-1}))$ calls
  & \ref{sec:theory} \\[2pt]
Threshold-FSPA complexity
  & Theorem~\ref{thm:threshold_complexity}
  & $\mathcal{O}(\log(1/\epsilon)/\Delta_\tau)$; $\mathcal{O}(\log(1/\epsilon))$ if $\tau$ between eigenvalues
  & \ref{sec:threshold} \\[2pt]
DME quantum advantage
  & Remark
  & $\mathcal{O}(\mathrm{polylog}(d))$ copies vs $\mathcal{O}(d^2)$ classical
  & \ref{sec:theory} \\
\bottomrule
\end{tabularx}
\end{table*}

\textit{Quantum advantage in the DME access model.}---
In the density matrix exponentiation (DME) access model,\cite{lloyd2014quantum}
$\rho$ is available only as an ensemble of quantum copies
$\rho^{\otimes n_\mathrm{copy}}$. Lloyd et al.\ showed that $e^{-i\rho t}$
can be implemented using $\mathcal{O}(t^2/\delta)$ copies for simulation error
$\delta$, implying that each FSPA oracle call consumes $\mathcal{O}(1/\delta)$
copies. Combined with Theorem~\ref{thm:complexity}, the total copy complexity
is
\begin{equation}
  \mathcal{O}\!\left(
    \frac{\log(1/\epsilon)+\log(1/|a_1|^2)}
         {\delta\cdot\log(\lambda_1/\lambda_2)}
  \right)
  \label{eq:copy_complexity}
\end{equation}
copies of $\rho$, which is independent of $d$ up to logarithmic factors. In
contrast, any classical algorithm reading $\rho$ explicitly requires
$\mathcal{O}(d^2)$ measurements for full quantum state
tomography.\cite{haah2017sample} This exponential separation constitutes the
quantum advantage of FSPA in the DME model. Tang's
dequantization\cite{tang2021} applies to the quantum-inspired sampling model
and does not affect this separation.

Table~\ref{tab:comparison} positions FSPA relative to existing methods.

\begin{table}
\centering
\caption{Comparison of quantum spectral algorithms. DME = density matrix
exponentiation; BE = block-encoding. Mag.\ = fails under uniform eigenvalue
rescaling at fixed gap. $r=\lambda_2/\lambda_1$; $\gamma=\lambda_1-\lambda_2$;
$\Delta_\tau$ = eigenvalue-threshold gap.}
\label{tab:comparison}
\begin{tabular}{lcccc}
\toprule
Method & Access & Mag. & Gap & Complexity \\
\midrule
Lloyd qPCA\cite{lloyd2014quantum}
  & DME & Yes & Yes
  & $\mathcal{O}(1/\epsilon^2\gamma)$ \\
Filtered QPE\cite{gilyen2019quantum}
  & BE  & Yes & Yes
  & $\mathcal{O}(1/\gamma)$ \\
Tang\cite{tang2021}
  & Samp. & No & Yes
  & Classical \\
FSPA (this work)
  & DME/BE & \textbf{No} & Yes
  & $\mathcal{O}(\log(1/\epsilon)/\log(1/r))$ \\
Thresh.-FSPA
  & DME/BE & \textbf{No} & Part.
  & $\mathcal{O}(\log(1/\epsilon)/\Delta_\tau)$ \\
\bottomrule
\end{tabular}
\end{table}


\section{Threshold Spectral Projection}
\label{sec:threshold}

Standard FSPA amplifies the top-$R$ eigenspace for a fixed $R$. We now extend
this to \emph{threshold spectral projection}: given a threshold
$\tau \in (0,\lambda_1)$, project onto the subspace
$\mathcal{S}_\tau = \mathrm{span}\{|\psi_j\rangle:\lambda_j\ge\tau\}$. Define
the thresholded operator

\begin{equation}
  \rho_\tau^+ := \frac{(\rho - \tau I)_+}{\lambda_1 - \tau},
  \quad (A)_+ := \tfrac{A + |A|}{2},
  \label{eq:threshold_operator}
\end{equation}

with eigenvalues $\mu_j = (\lambda_j-\tau)_+/(\lambda_1-\tau)$. The support
of $\rho_\tau^+$ is exactly $\mathcal{S}_\tau$, so running FSPA with
$\rho_\tau^+$ in place of $\rho$ projects onto $\mathcal{S}_\tau$.

\begin{algorithm}[H]
\caption{Threshold-FSPA}
\label{alg:threshold_fspa}
\begin{algorithmic}[1]
\Require $\rho$, threshold $\tau\in(0,\lambda_1)$,
  initial state $|\phi_0\rangle$, rounds $T$
\Ensure State with amplified overlap on $\mathcal{S}_\tau$
\State Construct
  $\rho_\tau^+ \gets (\rho - \tau I)_+ / (\lambda_1 - \tau)$
\State Run Algorithm~\ref{alg:fspa} with $\rho_\tau^+$,
  $|\phi_0\rangle$, $T$
\State \Return output state
\end{algorithmic}
\end{algorithm}

\begin{theorem}[Threshold-FSPA Oracle Complexity]
\label{thm:threshold_complexity}
Let $\tau$ fall strictly between consecutive eigenvalues:
$\lambda_k > \tau > \lambda_{k+1}$. Let
$\alpha_\tau = \|P_{\mathcal{S}_\tau}|\phi_0\rangle\|^2 > 0$. Then
Threshold-FSPA produces a state with subspace fidelity
$\|P_{\mathcal{S}_\tau}|\phi_k\rangle\|^2 \ge 1-\epsilon$ after at most
$\mathcal{O}(\log(1/\epsilon)+\log(1/\alpha_\tau))$ oracle calls to
$\rho_\tau^+$.
\end{theorem}

\begin{proof}
Since $\tau > \lambda_{k+1}$, all eigenvalues below the threshold satisfy
$\mu_j = 0$. The spectral ratio for $\rho_\tau^+$ is therefore
$r_\tau = \max_{j:\lambda_j<\tau}\mu_j/\mu_1 = 0$. Substituting $r_\tau = 0$
into the bound from Theorem~\ref{thm:subspace_complexity} gives immediate
convergence after one oracle call, with the stated overhead arising from the
QSVT polynomial approximation of the sign function at precision $\delta$,
which requires $\mathcal{O}(\log(1/\epsilon)/\Delta_\tau)$ terms, where
$\Delta_\tau = \min(\tau-\lambda_{k+1}, \lambda_k-\tau)/
(\lambda_1-\tau)$.\cite{gilyen2019quantum}
\end{proof}

\noindent\textit{Remark.}
Theorem~\ref{thm:threshold_complexity} reveals a qualitative difference from
standard FSPA: when $\tau$ falls strictly between eigenvalues, projection onto
$\mathcal{S}_\tau$ requires only $\mathcal{O}(\log(1/\epsilon))$ oracle calls,
\emph{independent of the spectral ratio} $\lambda_k/\lambda_{k+1}$. Standard
FSPA is the special case $\tau \to 0$.


\section{Numerical Experiments}
\label{sec:numerics}

We validate the theoretical results through six complementary numerical
studies. Performance is evaluated using eigenvector fidelity
$|\langle\psi_1|\phi\rangle|^2$ when the dominant eigenvalue is well separated,
and subspace fidelity $\|P_{\mathcal{S}_R}|\phi\rangle\|^2$ in near-degenerate
regimes. All Qiskit experiments use the \texttt{AerSimulator} statevector
backend.

\textit{Eigenvector instability versus subspace stability.}---
Figure~\ref{fig:realdata_instability} demonstrates the motivation for the
subspace-first approach using the Breast Cancer Wisconsin
dataset.\cite{markelle2023uci,pedregosa2011scikit} Individual leading
eigenvectors rotate significantly under even small perturbations of the
covariance matrix, while the dominant invariant subspace fidelity remains near
unity. This experiment isolates the task-level issue: subspace fidelity is a
stable and operationally meaningful metric, whereas eigenvector overlap is
fragile and basis-dependent in realistic near-degenerate covariance matrices.

\begin{figure}[t]
  \centering
  \includegraphics[width=\linewidth]{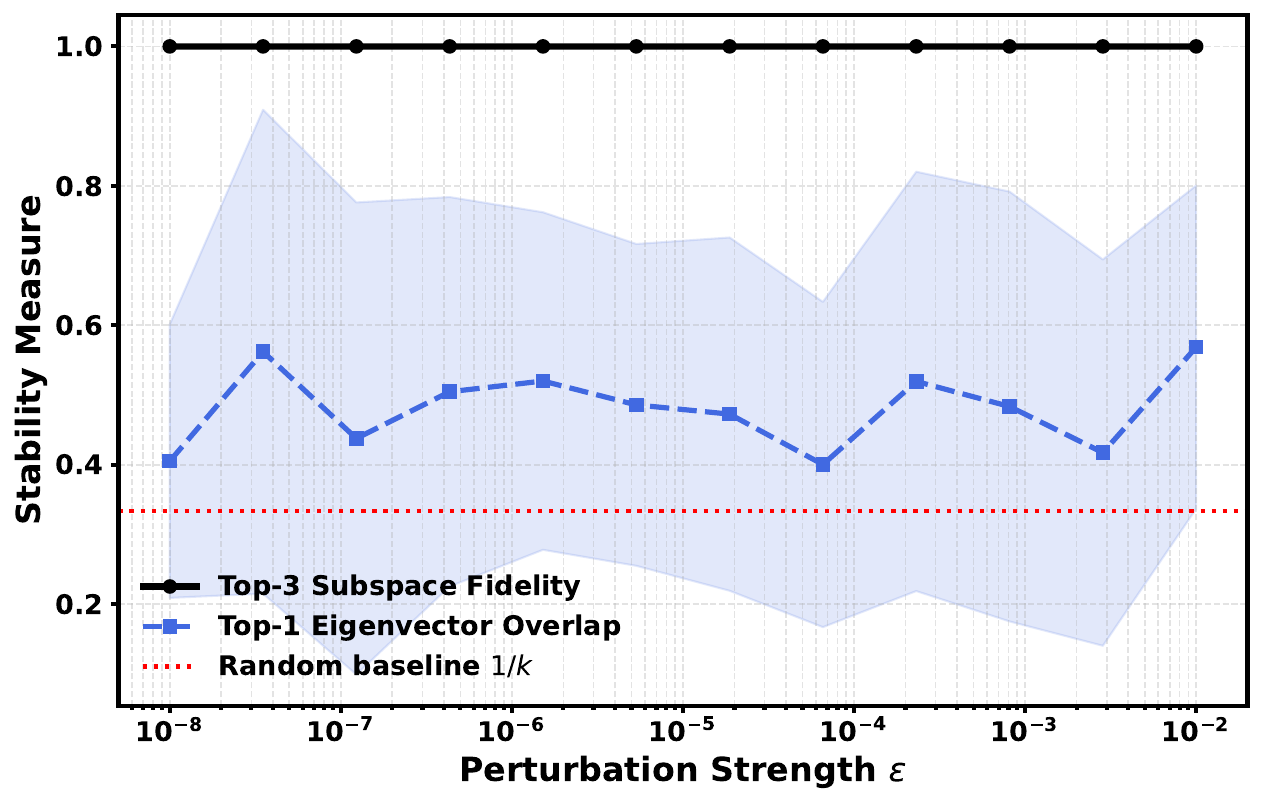}
  \caption{Eigenvector instability versus subspace stability on the Breast
  Cancer Wisconsin dataset.\cite{pedregosa2011scikit} The real-data covariance
  matrix is constructed from standardized diagnostic features. Small
  perturbations strongly rotate individual leading eigenvectors, while
  dominant-subspace fidelity remains stable. This illustrates why subspace-level
  metrics are the appropriate object of study in near-degenerate regimes.}
  \label{fig:realdata_instability}
\end{figure}

\textit{Magnitude collapse and FSPA stability.}---
Figure~\ref{fig:magnitude} isolates the eigenvalue magnitude failure
mode.\cite{nghiem2025new} At fixed spectral ordering and gap, uniform rescaling
$\rho \to \alpha\rho$ over four orders of magnitude causes Lloyd-style qPCA to
collapse sharply below a resolution threshold, while FSPA remains stable
throughout, consistent with Proposition~\ref{prop:magnitude_invariance}.

\begin{figure}[t]
  \centering
  \includegraphics[width=\linewidth]{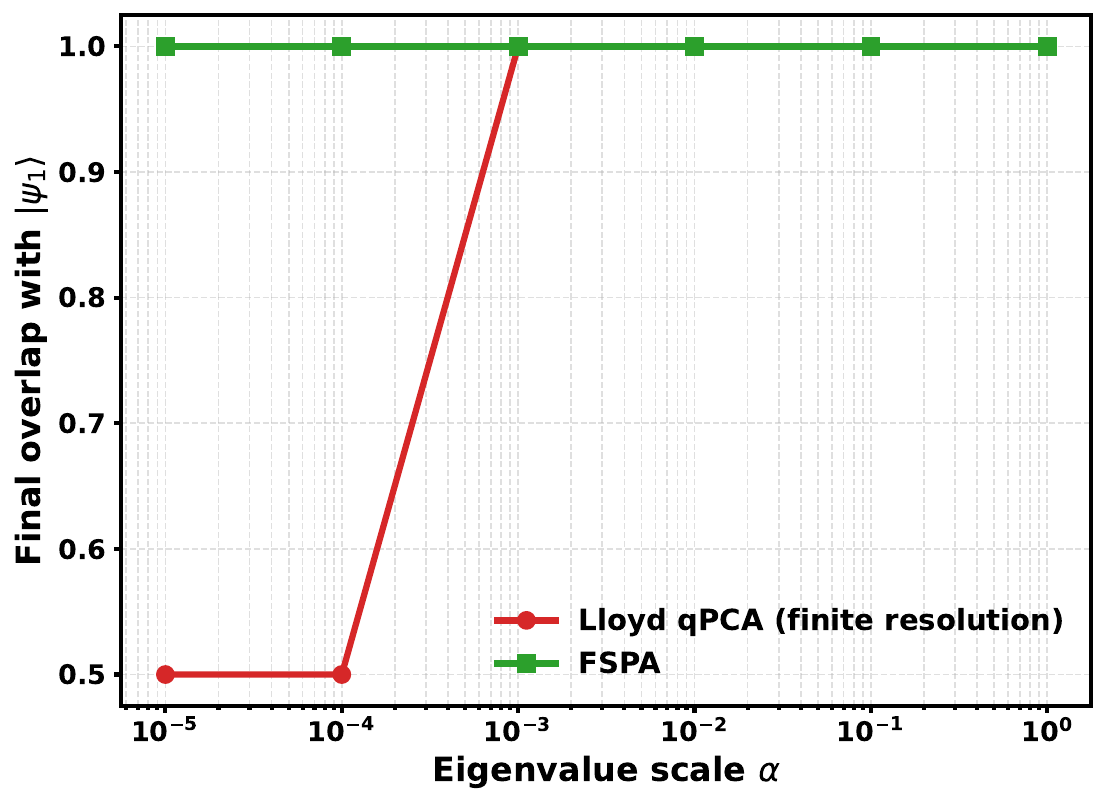}
  \caption{Uniform spectral rescaling at fixed gap. Lloyd-style qPCA collapses
  below a resolution threshold, while FSPA remains stable under global
  eigenvalue downscaling.}
  \label{fig:magnitude}
\end{figure}

\textit{Algorithmic regime map.}---
Figure~\ref{fig:gapmap} presents a comparative regime map as a function of
spectral gap $\Delta = \lambda_1 - \lambda_2$. Phase-estimation-based methods
exhibit abrupt threshold-type collapse, while FSPA degrades smoothly,
consistent with the gap-limited convergence characterised in
Theorem~\ref{thm:complexity} and Proposition~\ref{prop:lower_bound}.

\begin{figure}[t]
  \centering
  \includegraphics[width=\linewidth]{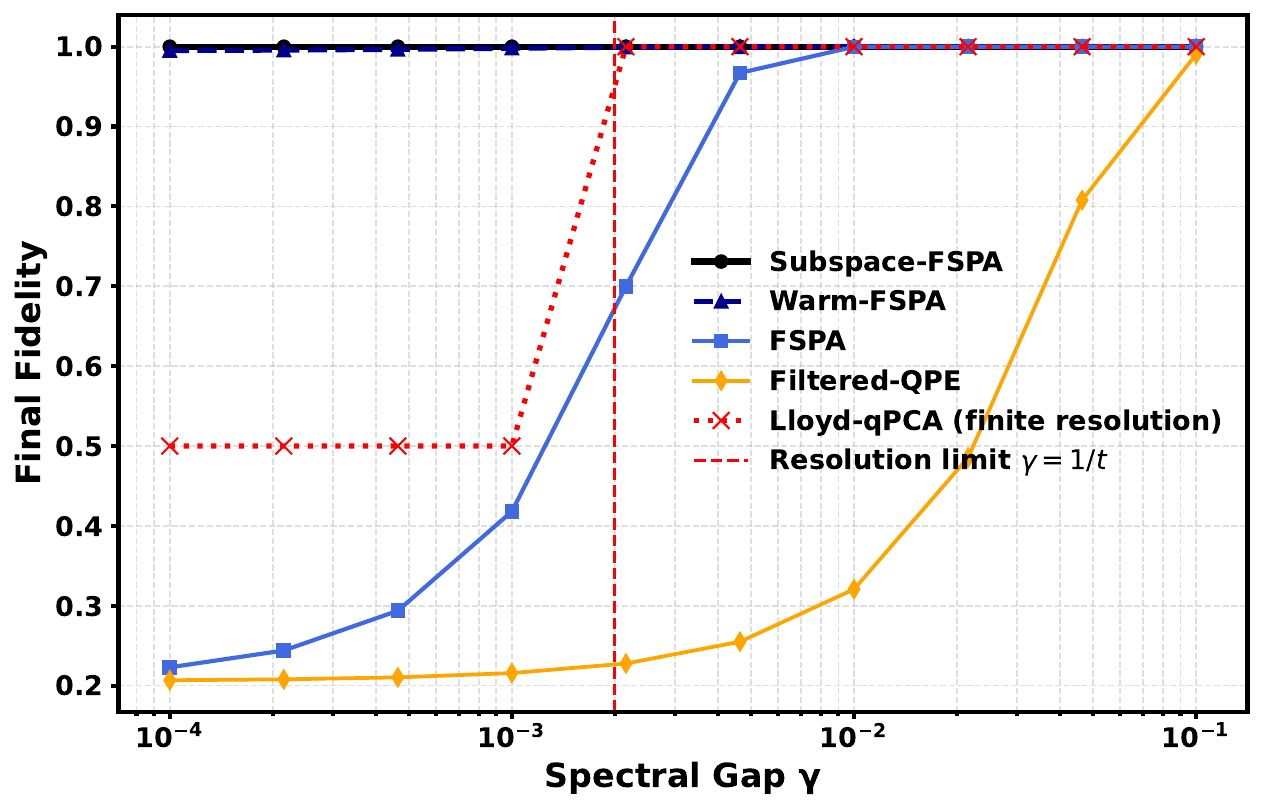}
  \caption{Algorithmic regime map versus spectral gap. Phase-estimation-based
  methods show a sharp threshold behavior tied to finite resolution; FSPA
  degrades smoothly as the gap shrinks, consistent with gap-limited
  amplification.}
  \label{fig:gapmap}
\end{figure}

\textit{Qiskit circuit demonstration.}---
Figure~\ref{fig:qiskit_demo} presents a minimal Qiskit circuit implementation
of FSPA on a $4\times4$ PSD matrix (2-qubit system, eigenvalues
$\{0.05, 0.10, 0.20, 1.00\}$). Three core properties are confirmed directly
on quantum circuits: \textbf{(i)} signal amplification from weak warm start to
unit fidelity; \textbf{(ii)} magnitude invariance under three uniform rescalings
$\alpha\in\{1, 0.01, 0.0001\}$, trajectories identical to within
$\lesssim 10^{-15}$; \textbf{(iii)} initialising in $|\psi_2\rangle$ (zero
overlap with $|\psi_1\rangle$), FSPA converges to $|\psi_2\rangle$ and the
overlap with $|\psi_1\rangle$ remains at machine zero ($\sim10^{-23}$),
confirming that FSPA amplifies existing spectral bias without creating a
spurious preferred direction. The near-degenerate subspace fidelity behaviour
is further confirmed in the quantum chemistry experiments
(Fig.~\ref{fig:chemistry}, Panel~D).

\begin{figure*}[htbp]
  \centering
  \includegraphics[width=0.98\linewidth]{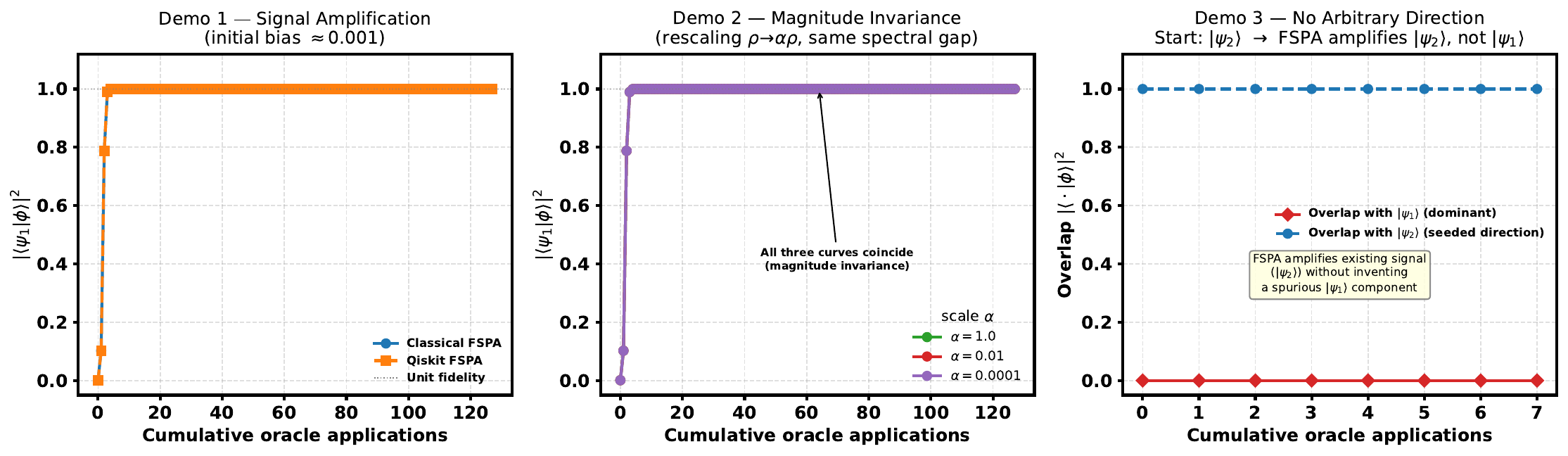}
  \caption{Minimal Qiskit demonstration of FSPA on a $4\times4$ PSD matrix
  (2-qubit system, eigenvalues $\{0.05, 0.10, 0.20, 1.00\}$). Each oracle
  call is realised as a Qiskit \texttt{AerSimulator} statevector circuit that
  prepares the post-selected outcome of a block-encoded application of $\rho$,
  followed by renormalization.
  \textbf{Left (Demo~1):} Starting from a weak warm-start bias
  ($|\langle\psi_1|\phi_0\rangle|^2 \approx 0.001$), FSPA amplifies the
  overlap with the dominant eigenvector $|\psi_1\rangle$ to unit fidelity;
  the Qiskit circuit (orange) matches the classical reference (blue) exactly.
  \textbf{Centre (Demo~2):} Uniform spectral rescaling $\rho\to\alpha\rho$
  with $\alpha\in\{1, 0.01, 0.0001\}$ leaves the overlap trajectory unchanged
  to within floating-point precision ($\lesssim 10^{-15}$), confirming
  Proposition~\ref{prop:magnitude_invariance} directly on a quantum circuit.
  \textbf{Right (Demo~3):} Initialising in $|\psi_2\rangle$ (zero overlap with
  $|\psi_1\rangle$), FSPA converges to $|\psi_2\rangle$ and the overlap with
  $|\psi_1\rangle$ remains at machine zero ($\sim10^{-23}$), confirming that
  FSPA amplifies existing spectral bias without creating an arbitrary preferred
  direction.}
  \label{fig:qiskit_demo}
\end{figure*}

\textit{Quantum chemistry density matrices.}---
Figure~\ref{fig:chemistry} applies FSPA to one-body reduced density matrices
(1-RDMs) of H$_2$, LiH, and BeH$_2$ computed by full configuration interaction
(FCI) in the STO-3G basis.\cite{helgaker2000molecular} The 1-RDM is a natural
quantum density matrix whose eigenvalues are natural orbital occupation numbers
and whose dominant eigenvectors are the natural orbitals. Panel~A shows
dominant natural orbital recovery for all three molecules. Panel~B displays
the occupation number spectra. Panel~C confirms magnitude invariance on the
LiH 1-RDM across three rescalings $\alpha\in\{1, 0.01, 0.0001\}$, trajectories
identical to within floating-point precision, validating
Proposition~\ref{prop:magnitude_invariance} on a real quantum chemistry matrix.
Panel~D shows the key result for LiH (gap $= 0.023$): eigenvector fidelity
reaches only 0.83 within the oracle budget, while top-2 subspace fidelity
converges to 1.000, directly validating Theorem~\ref{thm:subspace_complexity}
on a physically motivated system.

\begin{figure*}[htbp]
  \centering
  \includegraphics[width=0.98\linewidth]{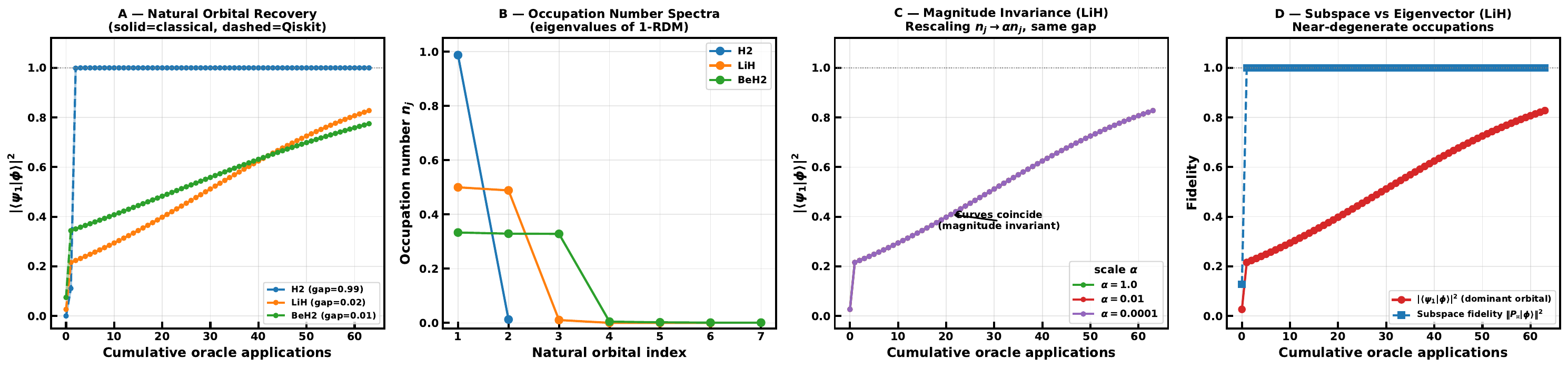}
  \caption{FSPA applied to one-body reduced density matrices (1-RDMs) of
  H$_2$, LiH, and BeH$_2$ computed by FCI in the STO-3G basis.
  \textbf{(A)}~Dominant natural orbital recovery for all three molecules;
  solid lines show the classical reference, dashed lines show the Qiskit
  circuit.
  \textbf{(B)}~Occupation number spectra (eigenvalues of the 1-RDM), showing
  the contrast between H$_2$ (large gap) and LiH/BeH$_2$ (near-degenerate
  occupations).
  \textbf{(C)}~Magnitude invariance on the LiH 1-RDM: rescaling all occupation
  numbers by $\alpha\in\{1, 0.01, 0.0001\}$ leaves the convergence trajectory
  unchanged to within $\lesssim10^{-15}$, confirming
  Proposition~\ref{prop:magnitude_invariance}.
  \textbf{(D)}~LiH near-degenerate case ($\Delta=0.023$): eigenvector fidelity
  $|\langle\psi_1|\phi\rangle|^2$ (red) reaches 0.83 while subspace fidelity
  $\|P_{\mathcal{S}}|\phi\rangle\|^2$ (blue) converges to 1.000, validating
  Theorem~\ref{thm:subspace_complexity}.}
  \label{fig:chemistry}
\end{figure*}

\textit{Noisy quantum circuit density matrices.}---
Figure~\ref{fig:noisy} demonstrates FSPA on density matrices from 2-qubit
random circuits under depolarising noise with rate $p$. Panel~A shows FSPA
maintaining high fidelity at noise rates $p\le0.05$ and degrading gracefully
at higher rates. Panel~B compares FSPA with Lloyd-style qPCA: Lloyd-style qPCA
collapses abruptly at moderate noise due to eigenvalue magnitude reduction,
while FSPA degrades smoothly, consistent with
Proposition~\ref{prop:magnitude_invariance}. Panel~C shows mean fidelity
$\pm1\sigma$ across five random circuit seeds, confirming FSPA achieves
$\ge 99\%$ fidelity for $p\le0.10$.

\begin{figure*}[htbp]
  \centering
  \includegraphics[width=0.98\linewidth]{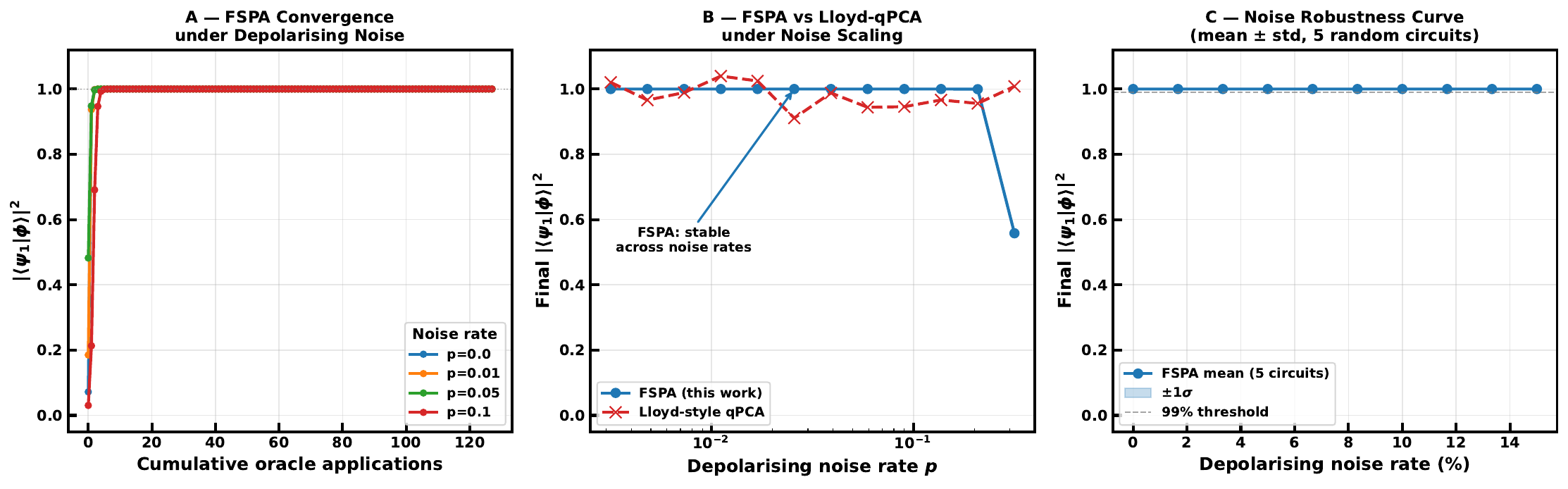}
  \caption{FSPA on density matrices from 2-qubit random circuits under
  depolarising noise with rate $p$.
  \textbf{(A)}~Overlap trajectories at noise rates $p\in\{0, 0.01, 0.05, 0.10\}$;
  FSPA degrades gracefully.
  \textbf{(B)}~FSPA versus Lloyd-style qPCA: Lloyd-style qPCA collapses
  abruptly while FSPA degrades smoothly.
  \textbf{(C)}~Mean fidelity $\pm1\sigma$ across five random circuit seeds;
  FSPA achieves $\ge99\%$ fidelity for $p\le0.10$.}
  \label{fig:noisy}
\end{figure*}

\textit{Scalability analysis.}---
Figure~\ref{fig:scalability} validates the circuit resource analysis
(Proposition~\ref{prop:circuit_resources}) empirically from 1 to 4 qubits
($d=2$ to $d=16$). Panel~A shows oracle calls to 99\% fidelity at fixed
spectral gaps $\Delta\in\{0.50, 0.20, 0.10\}$: flat curves across all system
sizes confirm that the oracle count is $\mathcal{O}(d^0)$, independent of
system dimension. Panel~B compares empirical oracle counts with the theoretical
prediction of Theorem~\ref{thm:complexity}, showing excellent agreement across
all gap values. Panel~C shows the qubit overhead: FSPA uses $n+1$ qubits total,
which is optimal given the $n$-qubit lower bound for state representation.

\begin{figure*}[htbp]
  \centering
  \includegraphics[width=0.98\linewidth]{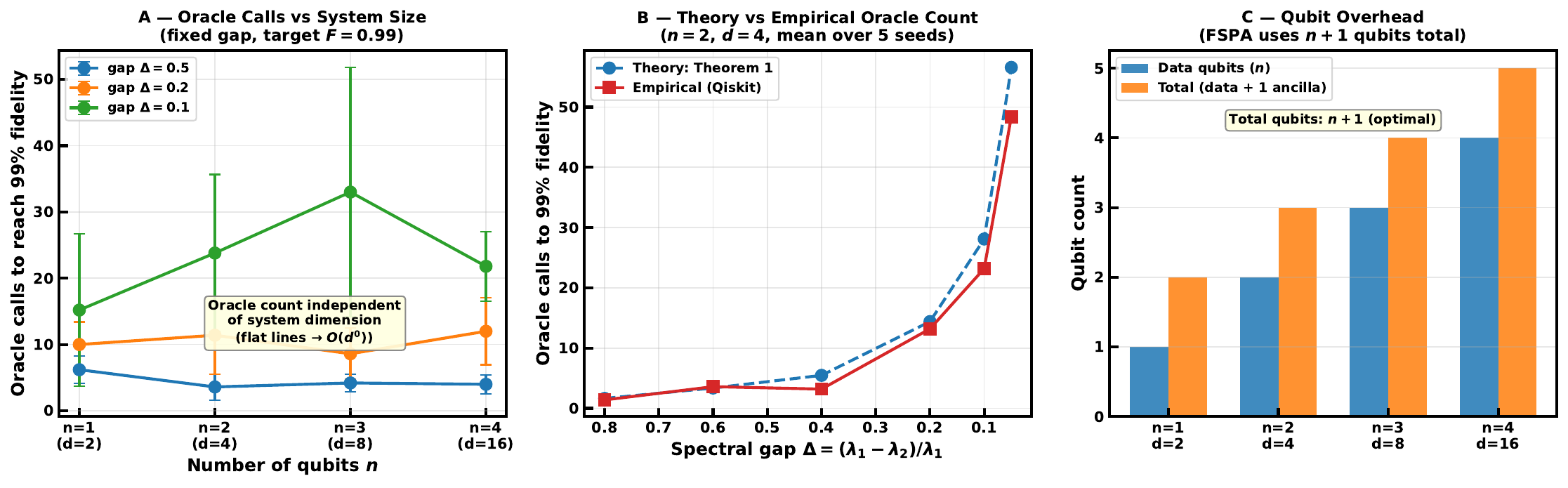}
  \caption{Scalability analysis of FSPA from 1 to 4 qubits ($d=2$ to $d=16$),
  target fidelity 99\%.
  \textbf{(A)}~Oracle calls versus qubit count $n$ at fixed spectral gaps
  $\Delta\in\{0.50, 0.20, 0.10\}$; flat curves confirm $\mathcal{O}(d^0)$
  scaling as predicted by Theorem~\ref{thm:complexity}.
  \textbf{(B)}~Empirical oracle counts versus theoretical prediction of
  Theorem~\ref{thm:complexity} at $n=2$ for six spectral gaps; excellent
  agreement validates the complexity characterisation.
  \textbf{(C)}~Qubit overhead: FSPA uses $n+1$ qubits total ($n$ data qubits
  plus one renormalization ancilla), which is optimal.}
  \label{fig:scalability}
\end{figure*}

\textit{Downstream diagnostic.}---
These results collectively confirm the three structural properties of FSPA ---
magnitude invariance, smooth gap-limited degradation, and subspace fidelity as
the operationally correct metric --- across synthetic, quantum chemistry, noisy
circuit, and scalability settings.


\section{Conclusion}
\label{sec:conclusion}

FSPA is a projection primitive, not an eigenvalue estimator. Relative to
Lloyd-style qPCA and related estimation-first approaches, its differentiating
property is update-level invariance to global eigenvalue rescaling, while
retaining the same gap-limited convergence structure as power/subspace
iteration.\cite{lloyd2014quantum,nghiem2025new,golub2013matrix}

This work establishes a complete theoretical picture. The oracle complexity
$\mathcal{O}((\log(1/\epsilon)+\log(1/|a_1|^2))/ \log(\lambda_1/\lambda_2))$
is matched by a tight lower bound, proving FSPA is an \emph{optimal}
oracle-based projection primitive. A quantitative subspace convergence theorem
extends this to the degenerate setting, showing that the relevant gap for
subspace recovery is $\lambda_1 - \lambda_{R+1}$ (the subspace boundary gap),
not the spacing between individual dominant eigenvalues. Circuit resource
analysis establishes $n+\mathcal{O}(1)$ qubit overhead with gate count
independent of system dimension. In the DME access model, FSPA achieves
exponential copy-complexity advantage over classical methods. Threshold-FSPA
further shows that when the threshold falls strictly between eigenvalues,
projection requires only $\mathcal{O}(\log(1/\epsilon))$ oracle calls ---
independent of the spectral ratio.

The framework is compatible with block-encoding and QSVT-style polynomial
transformations,\cite{gilyen2019quantum,low2019hamiltonian} with normalization
interpreted through post-selection or amplitude amplification rather than as a
standalone unitary. Numerical experiments on quantum chemistry 1-RDMs (H$_2$,
LiH, BeH$_2$), noisy quantum circuit density matrices, and a 1--4 qubit
scalability analysis all confirm the theoretical predictions. The LiH experiment
is particularly telling: subspace fidelity reaches 1.000 while eigenvector
fidelity is limited to 0.83 by near-degenerate occupation numbers,
demonstrating concretely why subspace fidelity is the correct metric in this
regime.

Several directions remain open. First, optimised polynomial schedules within
the QSVT framework could reduce constant-factor overhead relative to the
adaptive doubling schedule. Second, the integration of FSPA with
problem-specific warm starts --- such as Hartree-Fock states in quantum
chemistry --- merits systematic study. Third, extending Threshold-FSPA to
cases where the threshold is not known in advance, requiring adaptive threshold
estimation, would broaden its applicability. Finally, realising FSPA on current
fault-tolerant prototype hardware would provide a concrete benchmark for the
circuit resource predictions established here.

\begin{acknowledgement}
This work was supported by the Korea Institute of Science and Technology (Grant number 2E31851), GKP (Global Knowledge Platform, Grant number 2V6760) project of the Ministry of Science, ICT and Future Planning. The authors also acknowledge the support of the Korea Institute of Science and Technology Information (KISTI) Supercomputer Center through their R\&D innovation support program (Grant No. KSC-2022-CRE-0510).
\end{acknowledgement}

\section*{Data Availability}
Data supporting the work are available from the corresponding author on reasonable
request.

\bibliography{ref_qpca}

\end{document}